\documentclass{article}

\usepackage[letterpaper,textheight=9in,textwidth=5.5in,top=1in,
            headheight=12pt,headsep=25pt,footskip=30pt]{geometry}
\usepackage{natbib}
\makeatletter

\newcommand{\delequal}{\mathrel{\stackrel{\scriptscriptstyle\Delta}{=}}}

\renewcommand{\normalsize}{%
  \@setfontsize\normalsize\@xpt\@xipt
  \abovedisplayskip      7\p@ \@plus 2\p@ \@minus 5\p@
  \abovedisplayshortskip \z@ \@plus 3\p@
  \belowdisplayskip      \abovedisplayskip
  \belowdisplayshortskip 4\p@ \@plus 3\p@ \@minus 3\p@}
\normalsize
\renewcommand{\small}{%
  \@setfontsize\small\@ixpt\@xpt
  \abovedisplayskip      6\p@ \@plus 1.5\p@ \@minus 4\p@
  \abovedisplayshortskip \z@  \@plus 2\p@
  \belowdisplayskip      \abovedisplayskip
  \belowdisplayshortskip 3\p@ \@plus 2\p@   \@minus 2\p@}
\renewcommand{\large}{\@setfontsize\large\@xiipt{14}}
\renewcommand{\Large}{\@setfontsize\Large\@xivpt{16}}
\renewcommand{\LARGE}{\@setfontsize\LARGE\@xviipt{20}}
\renewcommand{\huge}{\@setfontsize\huge\@xxpt{23}}
\renewcommand{\Huge}{\@setfontsize\Huge\@xxvpt{28}}

\renewcommand{\scriptsize}{\fontsize{7pt}{8pt}\selectfont}

\renewcommand{\section}{\@startsection{section}{1}{\z@}%
  {-2.0ex \@plus -0.5ex \@minus -0.2ex}{1.5ex \@plus 0.3ex \@minus 0.2ex}%
  {\large\bf\raggedright}}
\renewcommand{\subsection}{\@startsection{subsection}{2}{\z@}%
  {-1.8ex \@plus -0.5ex \@minus -0.2ex}{0.8ex \@plus 0.2ex}%
  {\normalsize\bf\raggedright}}
\renewcommand{\subsubsection}{\@startsection{subsubsection}{3}{\z@}%
  {-1.5ex \@plus -0.5ex \@minus -0.2ex}{0.5ex \@plus 0.2ex}%
  {\normalsize\bf\raggedright}}
\renewcommand{\paragraph}{\@startsection{paragraph}{4}{\z@}%
  {1.5ex \@plus 0.5ex \@minus 0.2ex}{-1em}{\normalsize\bf}}

\renewcommand{\footnoterule}{\kern-3\p@ \hrule width 12pc \kern 2.6\p@}
\def\@listi{\leftmargin\leftmargini}
\renewcommand{\maketitle}{\par
  \begingroup
    \renewcommand{\thefootnote}{\fnsymbol{footnote}}
    \renewcommand{\@makefnmark}{\hbox to \z@{$^{\@thefnmark}$\hss}}
    \long\def\@makefntext##1{\parindent 1em\noindent
      \hbox to 1.8em{\hss $\m@th ^{\@thefnmark}$}##1}
    \thispagestyle{plain}\@maketitle\@thanks
  \endgroup
  \let\maketitle\relax\let\thanks\relax}
\renewcommand{\@maketitle}{\vbox{\hsize\textwidth\linewidth\hsize
    \vskip 0.1in
    \hrule height 4\p@ \vskip 0.25in \vskip -\parskip
    \centering{\LARGE\bf \@title\par}
    \vskip 0.29in \vskip -\parskip \hrule height 1\p@ \vskip 0.09in
    \def\And{\end{tabular}\hfil\linebreak[0]\hfil
      \begin{tabular}[t]{c}\bf\rule{\z@}{24\p@}\ignorespaces}%
    \def\AND{\end{tabular}\hfil\linebreak[4]\hfil
      \begin{tabular}[t]{c}\bf\rule{\z@}{24\p@}\ignorespaces}%
    \begin{tabular}[t]{c}\bf\rule{\z@}{24\p@}\@author\end{tabular}%
    \vskip 0.3in \@minus 0.1in}}
\renewenvironment{abstract}{\vskip 0.075in\centerline{\large\bf Abstract}%
  \vspace{0.5ex}\begin{quote}}{\par\end{quote}\vskip 1ex}
  
\makeatother
\usepackage[utf8]{inputenc}
\usepackage{amsmath,amssymb,amsthm,mathtools}
\usepackage{booktabs}
\usepackage{enumitem}
\usepackage{xcolor}
\usepackage{tikz}
\usetikzlibrary{arrows.meta,patterns,positioning}
\usepackage[hidelinks]{hyperref}

\newcommand{\Prob}{\mathbb{P}}
\newcommand{\yx}[2]{y_{{#1}_{x_{#2}}}}   
\newcommand{\cM}{\mathcal{M}}
\newcommand{\cT}{\mathcal{T}}
\newcommand{\cG}{\mathcal{G}}
\newcommand{\PNS}{\mathrm{PNS}}
\newcommand{\PSub}{\mathrm{PSub}}
\newcommand{\PRep}{\mathrm{PRep}}
\newcommand{\PN}{\mathrm{PN}}
\newcommand{\one}{\mathbf{1}}

\theoremstyle{plain}
\newtheorem{theorem}{Theorem}
\newtheorem{lemma}[theorem]{Lemma}
\newtheorem{proposition}[theorem]{Proposition}
\newtheorem{corollary}[theorem]{Corollary}
\theoremstyle{definition}
\newtheorem{definition}[theorem]{Definition}

\title{Unified Sharp Bounds on Probabilities of Causation}

\author{%
  Xin Shu \quad Ang Li\\
  Department of Computer Science, Florida State University\\
  \texttt{xs24a@fsu.edu, angli@cs.fsu.edu}
}

\begin{document}

\maketitle

\begin{abstract}

Probabilities of causation (PoCs) have developed steadily, from the binary measures first introduced by Pearl to the non-binary setting recently
studied by Shu et al. As the theory has matured, PoCs have found applications in
a growing range of areas, including personalized medicine, marketing and legal attribution. Shu et al. proposed closed-form bounds for multi-valued treatments and outcomes and proved them valid, but left their sharpness in high dimensions open. We close this gap with a mathematical proof. We further show that $\PNS$, $\PSub$, $\PRep$ and $\PN$
are all special cases of a single, more general formula. Combined with their equivalence classes, this settles the question of sharp bounds for discrete PoCs from experimental and observational data: in every dimension, they are given in closed form by a single theorem.
\end{abstract}

\section{Introduction}\label{sec:intro}

Many decisions turn on how a particular individual would respond to different
treatments, not on how the population responds on average. A retailer may want to know
whether a customer would buy at one price but not at another; a marketer may want to
know whether someone would answer an e-mail campaign but ignore a phone call. Questions
of this kind are probabilities of causation (PoCs) \citep{pearl1999probabilities}. They
concern a unit's outcomes under several treatments at once, and since we never observe
more than one of these for the same unit, PoCs are in general not point-identified.
Experimental data $\Prob(\yx{i}{j})$ and observational data $\Prob(x_j,y_i)$ only
determine an interval, and the quantity of interest is the \emph{sharp} interval, i.e.\
the set of all values consistent with the data.

When treatment and outcome are binary, the sharp bounds have long been known
\citep{tian2000probabilities}. They were derived with Balke's linear-programming method
\citep{balke1995probabilistic,balke1997bounds} and later proved analytically
\citep{li2019unit,li2022unit}. Once the variables take more than two values, things get
harder. Balke's program still applies, but it needs one variable for every pair of
response type and received treatment, which for $n$-valued variables means
$n^n\cdot n$ variables. It can therefore only be solved for small $n$, and it gives no
closed form. \citet{zhang2022partial} solve an optimisation problem for each instance,
and \citet{li2024probabilities} give recursive bounds that are valid but not sharp.
\citet{Shu_2026} derived closed-form bounds for four families of PoCs, $\PNS(k)$,
$\PSub(k,p)$, $\PRep(k,q)$ and $\PN(k,p,q)$, and showed that equivalence classes and
relabelling reduce general discrete PoCs to these four. They proved the bounds valid in
every dimension and checked numerically that they are sharp for three- and four-valued
variables, but whether they are sharp in general was left as a conjecture. The
difficulty is that a direct proof has to handle the vertices of a polytope whose
dimension grows faster than exponentially in $n$.

\paragraph{Contributions.}
\begin{itemize}[leftmargin=*,itemsep=1pt,topsep=2pt]
\item We prove that the closed-form bounds of \citet{Shu_2026} for $\PNS(k)$,
$\PSub(k,p)$, $\PRep(k,q)$ and $\PN(k,p,q)$ are sharp in every dimension.
\item We prove two lemmas that make the sharpness proof manageable. The first reduces
Balke's program from $n^n\cdot n$ variables to $2^k\cdot n$, and the second reduces the
$2^k$ candidate cuts for the lower bound to just $k+2$.
\item We show that all four families, together with their relabellings, are special
cases of one general formula, which gives the sharp bounds for any such query in
closed form.
\end{itemize}

Section~\ref{sec:main} gives the full proof for $\PNS(k)$, and extends it to the other three families. Section~\ref{sec:unified} combines all four into a single theorem.
\section{Preliminaries}\label{sec:prelim}
This section recalls the definitions of probabilities of causation and the notation of
structural causal models (SCMs) \citep{galles1998axiomatic,halpern2000axiomatizing}.
Readers who know SCMs can skip it.

In an SCM, counterfactuals have a precise meaning. The statement ``$Y$ would have been
$y$ had $X$ been $x$'' is written $Y_x=y$, where the subscript $x$ refers to a
hypothetical intervention. To keep the notation short, we write $y_x$, $y_{x'}$, $y'_x$
and $y'_{x'}$ for the events $Y_x=y$, $Y_{x'}=y$, $Y_x=y'$ and $Y_{x'}=y'$. Throughout,
$X$ is the treatment and $Y$ the outcome. Unless stated otherwise, experimental data
are causal effects such as $P(y_x)$, and observational data are joint probabilities
such as $P(x,y)$.

When $X$ and $Y$ are binary, we write $x$ and $x'$ for $X=\text{true}$ and
$X=\text{false}$, and $y$ and $y'$ for $Y=\text{true}$ and $Y=\text{false}$.
\citet{pearl1999probabilities} defined three basic probabilities of causation.

\begin{definition}[Probabilities of causation \citep{pearl1999probabilities}]
\[
\text{PN}\delequal P(y'_{x'}\mid x,y),\qquad
\text{PS}\delequal P(y_x\mid y',x'),\qquad
\text{PNS}\delequal P(y_x,y'_{x'}).
\]
\end{definition}

PNS is the probability that $y$ responds to $x$ in both directions: $y$ occurs with
$x$ and fails to occur without it. It therefore captures both the sufficiency and the
necessity of $x$ for $y$. \citet{tian2000probabilities} obtained tight bounds on PN, PS
and PNS with Balke's linear-programming method \citep{balke1995probabilistic}, and
\citet{li2019unit,li2022unit} later proved these bounds analytically.

Now let $X$ and $Y$ each take $n$ values. \citet{li2024probabilities} named four
multi-valued probabilities of causation, PNS, PSub, PRep and PN, and gave recursive
bounds for them. \citet{Shu_2026} kept the names but rewrote each quantity so that the
$j$-th counterfactual always pairs $y_j$ with $x_j$:
\begin{align*}
\mathrm{PNS}(k)&=P({y_1}_{x_1},\ldots,{y_k}_{x_k})
&&\text{(probability of necessity and sufficiency)},\\
\mathrm{PSub}(k,p)&=P({y_1}_{x_1},\ldots,{y_k}_{x_k},x_p)
&&\text{(probability of substitute)},\\
\mathrm{PRep}(k,q)&=P({y_1}_{x_1},\ldots,{y_k}_{x_k},y_q)
&&\text{(probability of replacement)},\\
\mathrm{PN}(k,p,q)&=P({y_1}_{x_1},\ldots,{y_k}_{x_k},x_p,y_q)
&&\text{(probability of necessity)},
\end{align*}
with $k\le n$ and $p\ne j$ for $1\le j\le k$. Here $\mathrm{PNS}(k)$ is the probability that the same individual would have outcome $y_j$ under $x_j$ for every $j=1,\dots,k$. \citet{Shu_2026} derived closed-form bounds for all four quantities, and by their equivalence-class and replaceability theorems, any discrete PoC can be brought into one of these four forms. We refer to \citet{Shu_2026} for the bounds and the two theorems, in this paper we prove that the bounds are sharp.

\section{Notation and Definition}\label{sec:setup}
The treatment $X$ takes values $x_1,\dots,x_n$ and the outcome $Y$ takes values $y_1,\dots,y_n$; as in \citet{Shu_2026}, unequal cardinalities are handled by adding values of probability zero. We write $[n]=\{1,\dots,n\}$. The index $i$ runs over outcomes and $j,\ell$ over treatments; $p$ denotes an additional treatment and $q$ an additional outcome.

We work in structural causal models without further structural assumptions, and we reason with events throughout. With a slight abuse of notation, we write $x_j$ for the event $\{X=x_j\}$, i.e.\ that the unit received treatment $x_j$, and $y_i$ for the event $\{Y=y_i\}$, i.e.\ that the observed outcome is $y_i$. The counterfactual event $\yx{i}{j}$ means ``under $x_j$ the outcome would be $y_i$''; for fixed $j$ these events form a partition. If we arrange the observational probabilities $\Prob(x_j,y_i)$ in a table with one column per treatment and one row per outcome, the event $x_j$ is the $j$-th \emph{column} and each entry $Q_{ij}:=(x_j,y_i)$ is a \emph{cell}; as for $b_{ij}$, the first index of $Q$ is the row and the second the column. The columns partition the sample space, so $\Prob(E)=\sum_j\Prob(E,x_j)$ for every event $E$.

We now introduce new notations for probabilities:
\begin{equation}
a_{ij}=\Prob(\yx{i}{j})\ \ \text{(experimental)},\quad
b_{ij}=\Prob(x_j,y_i)\ \ \text{(observational)},\quad
c_j=\Prob(x_j)=\textstyle\sum_ib_{ij},
\end{equation}
The two kinds of data are linked by \emph{consistency} \citep{galles1998axiomatic, pearl2009causality, pearl2010consistency}: a unit that received $x_j$ has outcome $y_i$ exactly when its potential outcome under $x_j$ is $y_i$, i.e.\ $x_j\cap\yx{i}{j}=x_j\cap y_i$. Hence
\begin{equation}\label{eq:consistency}
b_{ij}=\Prob(Q_{ij})=\Prob\big(x_j, \yx{i}{j}\big).
\end{equation}

A \emph{compatible model} is a joint law of a response type
$R=(R_1,\dots,R_n)\in[n]^n$, where $R_j:=Y_{x_j}$, and a received treatment $X\in[n]$, with $\Prob(R_j=i)=a_{ij}$ and $\Prob(R_j=i,X=j)=b_{ij}$; the observed outcome is $R_X$. Let $\cM$ be the set of compatible models, assumed non-empty.

A value $r=(r_1,\dots,r_n)\in[n]^n$ of $R$ is also called a response type; it lists an individual's outcome under every treatment ($r_\ell=i$ means $\yx{i}{\ell}$ holds). It is a function from treatments to outcomes, so there are $n^n$ types.

Let's fix $1\le k\le n$ and put
\begin{align}\label{eq:target}
&A_j:=\yx{j}{j},\quad &&S:=\textstyle\bigcap_{j\le k}A_j,\nonumber\\
&\alpha_j:=a_{jj},\quad &&\beta_j:=b_{jj},\nonumber\\
&d_j:=\alpha_j-\beta_j,\quad &&s_j:=1-c_j-d_j .
\end{align}

By \eqref{eq:consistency}, $A_j\cap x_j$ \emph{is} the observed cell $Q_{jj}=(x_j,y_j)$, so $\beta_j$ is the only observed part of $A_j$. But $A_j$ is defined for every individual, whichever treatment it actually received, so it also has mass outside column $x_j$. This mass is $d_j=\Prob(A_j\setminus x_j)=\Prob(A_j)-\Prob(A_j,x_j)$, and it is spread over the other columns $x_\ell$, $\ell\ne j$. We write
\[
m_{j\ell}:=\Prob(A_j,x_\ell)
\]
for the mass of $A_j$ in column $x_\ell$. Thus $m_{jj}=\beta_j$,
$\sum_{\ell\ne j}m_{j\ell}=d_j$ and $0\le m_{j\ell}\le c_\ell$. Only the diagonal entry $m_{jj}$ is observed; for $\ell\ne j$, $m_{j\ell}$ is not determined by the data.

Each of the four families $\PNS$, $\PSub$, $\PRep$ and $\PN$ has the form $\theta=\Prob(S, G)$ for some factual event $G$, i.e.\ a union of cells:
\begin{align}\label{eq:families}
&\PNS(k){:}\ G=\Omega, &&\PSub(k,p){:}\ G=x_p,\ \nonumber\\ 
&\PRep(k,q){:}\ G=y_q, &&\PN(k,p,q){:}\ G=(x_p,y_q),
\end{align}
with $p>k$. Here $\Omega$ denotes the whole sample space. Queries with other target outcomes reduce to these by relabelling. Let $\Theta:=\{\theta(M):M\in\cM\}$.

Let $\cG:=\{(i,\ell):Q_{i\ell}\subseteq G\}$ be the set of cells that make up the factual event $G$ (all cells for $\PNS$, column $p$ for $\PSub$, row $q$ for $\PRep$, the single cell $Q_{qp}$ for $\PN$); a pair $(i,\ell)\in\cG$ refers to row $i$ (the observed
outcome) and column $\ell$ (the received treatment). For a response type $r\in[n]^n$, let
\[
W(r):=\{j\le k:\ r_j=j\}\subseteq[k]
\]
record which of the targets $A_1,\dots,A_k$ it satisfies. Call a triple $(W,i,\ell)$, with $W\subseteq[k]$ and $i,\ell\in[n]$, \emph{consistent} if, whenever $\ell\le k$, $\ell\in W\iff i=\ell$. This is consistency: a unit that received $x_\ell$ ($\ell\le k$) satisfies $A_\ell$ if and only if its observed outcome is $y_\ell$.

\begin{lemma}\label{lem:star}
For every $j\le k$,
\begin{equation}\tag{$\star$}\label{eq:star}
s_j=1-c_j-d_j\ \ge\ 0 .
\end{equation}
\end{lemma}

\begin{proof}
Since $b_{ij}=\Prob(x_j,\yx{i}{j})\le\Prob(\yx{i}{j})=a_{ij}$, $\sum_i a_{ij}=1$ and
$\sum_i b_{ij}=c_j$,
\[
s_j=(1-a_{jj})-(c_j-b_{jj})=\sum_{i\ne j}a_{ij}-\sum_{i\ne j}b_{ij}
=\sum_{i\ne j}(a_{ij}-b_{ij})\ \ge\ 0. \qedhere
\]
\end{proof}

This inequality is the key to the lower bound: in Lemma~\ref{lem:small} it rules out all candidate lower bounds except the three in Theorem~\ref{thm:pns}.


\begin{definition}\label{def:admissible}
An \emph{admissible system} is a probability space with a partition into cells $Q_{i\ell}$, $i,\ell\in[n]$, where $Q_{i\ell}$ plays the role of the observed cell $(x_\ell,y_i)$ and has $\Prob(Q_{i\ell})=b_{i\ell}$; columns $x_\ell:=\bigcup_iQ_{i\ell}$; and events $A_1,\dots,A_k$ with $\Prob(A_j)=\alpha_j$ and $A_j\cap x_j=Q_{jj}$, the last condition playing the role of consistency. The \emph{value} of an admissible system is $\theta=\Prob(S, G)$, computed in that system, with $G$ read as the corresponding union of cells (i.e. $G = x_p$ in $\PSub(k,p)$). As before, $m_{j\ell}=\Prob(A_j,x_\ell)$ denotes the mass of $A_j$ in column $x_\ell$.
\end{definition}

\paragraph{The two programs.} Recall that $\cG=\{(i,\ell):Q_{i\ell}\subseteq G\}$ is the set of cells that make up the factual event $G$ (all cells for $\PNS$, column $p$ for $\PSub$, row $q$ for $\PRep$, the single cell $Q_{qp}$ for $\PN$). A compatible model is determined by the joint probabilities $\Prob(R=r,X=\ell)$ \citep{balke1997bounds,tian2000probabilities}. 

\emph{Balke's program} treats each of them as a variable $h_1(r,\ell)\ge0$, one for every response type $r\in[n]^n$ and received treatment $\ell\in[n]$, i.e.\ $n^n\cdot n$ variables:
\begin{equation}\label{eq:balke-full}
\begin{aligned}
\max/\min\ & \ \theta(h_1):=\sum_{r:\,W(r)=[k]}\ \sum_{\ell}\one\{(r_\ell,\ell)\in\cG\}\,h_1(r,\ell)\\
\text{s.t.}\ & \ \sum_{r:\,r_j=i}\ \sum_{\ell}h_1(r,\ell)=a_{ij},\qquad
\sum_{r:\,r_j=i}h_1(r,j)=b_{ij}\qquad(i,j\in[n]).
\end{aligned}
\end{equation}

An experimental constraint sums over all received treatments, so it only involves the response type $r$. An observational constraint looks at a single column, so it also needs the received treatment $\ell$. We cannot drop $\ell$ because the observational data cannot be written down without it, and they do carry information. If treatment were assigned independently of the response type, as in a randomized trial, then $b_{ij}=c_j\,a_{ij}$ and the observational data would tell nothing new. They help only when the choice of treatment depends on the response type, i.e.\ under confounding.

The \emph{reduced program} has one variable $h_2(W,i,\ell)\ge0$ for every consistent triple, i.e.\ at most $2^k\cdot n^2$ variables:
\begin{equation}\label{eq:reduced-full}
\begin{aligned}
\max/\min\ & \ \theta(h_2):=\sum_{(i,\ell)\in\cG}h_2([k],i,\ell)\\
\text{s.t.}\ & \ \sum_{W\ni j}\ \sum_{i,\ell}h_2(W,i,\ell)=\alpha_j\quad(j\le k),\qquad
\sum_{W}h_2(W,i,\ell)=b_{i\ell}\quad(i,\ell\in[n]).
\end{aligned}
\end{equation}

Recall that $[n]=\{1,\dots,n\}$. 
The reduced program keeps the experimental constraints only for $a_{jj}$ ($j\le k$), and keeps all observational constraints $\sum_W h_2(W,i,\ell)=b_{i\ell}$. It drops every $a_{ij}$ with $i\ne j$, or when $j>k$. 

For $j\le k$, these constraints and consistency give three more identities. The column $x_j$ has mass $c_j$, the part of it inside $A_j$ has mass $\beta_j$, and the rest has mass $c_j-\beta_j$:

\begin{equation}\label{eq:derived}
\sum_W\sum_i h_2(W,i,j)=c_j,\
\sum_{W\ni j}\sum_i h_2(W,i,j)=b_{jj}=\beta_j,\
\sum_{W\not\ni j}\sum_i h_2(W,i,j)=c_j-\beta_j .
\end{equation}

The feasible set of \eqref{eq:balke-full} is a polytope and $\theta$ is linear, so $\Theta=[\min\Theta,\max\Theta]$ and the sharp bounds are the two optima of \eqref{eq:balke-full}. \citet{Shu_2026} used this program to check sharpness numerically for $n\le4$.

\begin{lemma}[Reduction]\label{lem:reduction}
For each of the four families, Balke's program~\eqref{eq:balke-full} and reduced program~\eqref{eq:reduced-full} attain
exactly the same set of values of $\theta$.
In particular they have the same maximum and the same minimum, so the sharp bounds can be computed from the reduced program. Equivalently, $\Theta$ equals the set of values of admissible systems.
\end{lemma}

The last sentence holds because they both describe the same thing. In an admissible system, the events $A_1,\dots,A_k$ cut each cell $Q_{i\ell}$ into pieces. Each piece is labelled by the set $W\subseteq[k]$ of events that contain it, and $h_2(W,i,\ell)$ is its mass. Conversely, cutting the cells according to a feasible $h_2$ gives an admissible system, see Appendix.


\begin{proof}[Proof idea]
Both inclusions are constructive (see Appendix). \emph{Merging}: given
a feasible $h_1$, grouping the response types by $W(r)$ and by the observed outcome gives
a feasible $h_2$ with the same value. \emph{Splitting}: given a feasible $h_2$, we keep
$W$, the observed outcome and the received treatment of each block fixed, and draw every
other coordinate $r_j$ independently from an explicit distribution built from the
data. The constraints of \eqref{eq:reduced-full} are exactly what makes the result
satisfy all constraints of \eqref{eq:balke-full}.
\end{proof}


\begin{corollary}\label{cor:functionals}
For $\PNS(k)$, $\cG$ contains all cells, so only $h_2(W,\ell):=\sum_ih_2(W,i,\ell)$ matters, and the reduced program is equivalent to
\begin{equation}\label{eq:reduced}
\begin{aligned}
\max/\min\ & \ \textstyle\sum_{\ell}h_2([k],\ell)\\
\text{s.t.}\ & \ \textstyle\sum_{W\ni j}\sum_{\ell}h_2(W,\ell)=\alpha_j,\\
&\sum_{W\ni j}h_2(W,j)=\beta_j\quad(j\le k),\ \\
&\textstyle\sum_{W}h_2(W,\ell)=c_\ell\quad(\ell\in[n]),
\end{aligned}
\end{equation}
with $2^k\cdot n$ variables. The sharp bounds on $\PNS(k)$ are therefore functions of $(\alpha_1,\dots,\alpha_k,$ $\beta_1,\dots,$ $\beta_k,c_1,\dots,c_n)$ alone, data sharing these $2k+n$ numbers yield the same program as in~\eqref{eq:reduced}.
\end{corollary}

The proof of Corollary~\ref{cor:functionals} is in Appendix. The corollary also shortens the equivalence-class arguments of
\citet{Shu_2026}: adding outcome or treatment values of probability zero leaves
$(\alpha,\beta,c)$ unchanged up to empty columns.

\paragraph{Sharpness principle.} Soundness and attainability together give sharpness.

\begin{lemma}[Sharpness]\label{lem:principle}
If every admissible system has value $\le U$ and equality holds for at least one
admissible system, then $\max\Theta=U$ is a sharp upper bound. The same holds with
$\ge L$, and $\min\Theta=L$ is a sharp lower bound.
\end{lemma}

\begin{proof}
By Lemma~\ref{lem:reduction}, $U$ is a valid upper bound: $\theta(M)\le U$ for all
$M\in\cM$. Suppose some $U'$ satisfied (H1) $\theta(M)\le U'$ for every $M\in\cM$ and
(H2) $U'<U$. The attaining system gives, via Lemma~\ref{lem:reduction}, a model
$M_0\in\cM$ with $\theta(M_0)=U$, and (H1) applied to $M_0$ gives $U\le U'<U$, a
contradiction. The lower bound is symmetric.
\end{proof}

\section{Main results}\label{sec:main}
In this section, we first prove sharpness of the bounds on $\PNS(k)$ and then extend it to the other three families.

\subsection{Problem Statement}
Soundness of the bounds for all four families, $\PNS$, $\PSub$, $\PRep$ and $\PN$, was proved by \citet{Shu_2026}. We prove completeness, which together with soundness shows that the bounds are sharp.

We restate the $\PNS(k)$ theorem in our notation below.

\begin{theorem}[$\PNS(k)$]\label{thm:pns}
For $\theta=\Prob(\yx{1}{1},\dots,\yx{k}{k})$, $\Theta=[L,U]$, where
\begin{align*}
L&=\max\Big\{\underbrace{0}_{L_1},\ \underbrace{\textstyle\sum_{j\le k}\alpha_j-(k-1)}_{L_2},\
\underbrace{\max_{i\le k}\big[\beta_i+\textstyle\sum_{j\le k,\,j\ne i}(\alpha_j+c_j-\beta_j)-(k-1)\big]}_{L_3}\Big\},\\
U&=\min\Big\{\underbrace{\textstyle\sum_{j\le k}\beta_j+\sum_{\ell>k}c_\ell}_{U_1},\
\underbrace{\min_{j\le k}\alpha_j}_{U_2},\
\underbrace{\min_{I\subseteq[k],\,|I|\ge2}\tfrac{1}{|I|-1}\textstyle\sum_{j\in I}d_j}_{U_3}\Big\}.
\end{align*}
\end{theorem}

We write $U_3$ over subsets $I$ to make clear that its indices are distinct. A repeated index would give an invalid bound.
\subsection{The upper bound proof}\label{sec:upper}

By Lemma~\ref{lem:reduction} we can work with admissible systems, and by
Lemma~\ref{lem:principle} the proof splits into two parts, soundness (Corollary~\ref{cor:pns-sound}) and attainability (Lemma~\ref{lem:pns-attain}). Soundness was shown by \citet{Shu_2026}; we re-derive it here because the same three
properties drive the construction in Lemma~\ref{lem:pns-attain}. Throughout, let
$\theta_\ell:=\Prob(S,x_\ell)$ denote the part of $\theta$ that lies in column $\ell$.

\begin{proposition}\label{prop:pns-props}
Every admissible system satisfies:
\begin{enumerate}[label=\textup{(Prop\arabic*)},leftmargin=*,itemsep=0pt,topsep=2pt]
\item $\sum_{\ell=1}^n\theta_\ell=\theta$;
\item $0\le\theta_j\le\beta_j$ for $j\le k$, and $0\le\theta_\ell\le c_\ell$ for $\ell>k$;
\item $\theta-\theta_j\le d_j$ for $j\le k$.
\end{enumerate}
\end{proposition}

\begin{proof}
(Prop1) The columns $x_1,\dots,x_n$ partition the space, so $S$ is the disjoint union of
the sets $S\cap x_\ell$. (Prop2) For $j\le k$, $S\subseteq A_j$ gives
$S\cap x_j\subseteq A_j\cap x_j$, which has mass $\beta_j$. For $\ell>k$,
$S\cap x_\ell\subseteq x_\ell$, which has mass $c_\ell$. (Prop3) By (Prop1),
$\theta-\theta_j=\Prob(S\setminus x_j)$. Since $S\subseteq A_j$,
$S\setminus x_j\subseteq A_j\setminus x_j$, which has mass $\alpha_j-\beta_j=d_j$.
\end{proof}

\begin{corollary}[Soundness]\label{cor:pns-sound}
Every admissible system satisfies $\theta\le U$.
\end{corollary}
\begin{proof}
We derive each of the three upper bounds in Theorem~\ref{thm:pns} from Proposition~\ref{prop:pns-props}.
\begin{itemize}[leftmargin=1.6em,itemsep=3pt,topsep=3pt]
\item \emph{$U_1$.} By (Prop1) and (Prop2),
$\theta=\sum_{j\le k}\theta_j+\sum_{\ell>k}\theta_\ell\le\sum_{j\le k}\beta_j+\sum_{\ell>k}c_\ell=U_1$.
\item \emph{$U_2$.} By (Prop3) and (Prop2), $\theta\le\theta_j+d_j\le\beta_j+d_j=\alpha_j$ for every
$j\le k$. Hence $\theta\le\min_{j\le k}\alpha_j=U_2$.
\item \emph{$U_3$.} Let $I\subseteq[k]$ with $|I|\ge2$. Summing (Prop3) over $j\in I$ gives
$|I|\,\theta-\sum_{j\in I}\theta_j\le\sum_{j\in I}d_j$. By (Prop1) and $\theta_\ell\ge0$ (Prop2),
$\sum_{j\in I}\theta_j\le\theta$. Hence $(|I|-1)\,\theta\le\sum_{j\in I}d_j$, i.e.\
$\theta\le\frac{1}{|I|-1}\sum_{j\in I}d_j$. Minimising over $I$ gives $\theta\le U_3$.
\end{itemize}
Therefore $\theta\le\min\{U_1,U_2,U_3\}=U$.
\end{proof}

\begin{lemma}[Column allocation]\label{lem:pns-alloc}
Let $\theta\ge0$, and for every column $j\in[n]$ put
\[
\underline\theta_j:=\begin{cases}(\theta-d_j)^+, & j\le k,\\ 0, & j>k,\end{cases}
\qquad
\overline\theta_j:=\begin{cases}\beta_j, & j\le k,\\ c_j, & j>k.\end{cases}
\]
There exist $\theta_1,\dots,\theta_n$ with $\theta_j\in[\underline\theta_j,\overline\theta_j]$ for every $j$ and $\sum_{j=1}^n\theta_j=\theta$ if and only if $\theta\le U$.
\end{lemma}
\begin{proof}
The interval $[\underline\theta_j,\overline\theta_j]$ is exactly the range allowed for $\theta_j$ by
Proposition~\ref{prop:pns-props}: the upper end comes from (Prop2), and for $j\le k$ the lower end
comes from (Prop3), $\theta_j\ge\theta-d_j$, together with $\theta_j\ge0$.

If every interval is non-empty, the sum $\sum_j\theta_j$ with $\theta_j\in[\underline\theta_j,\overline\theta_j]$
takes exactly the values in $\big[\sum_j\underline\theta_j,\ \sum_j\overline\theta_j\big]$. Hence a solution
exists if and only if the following three conditions hold.
\begin{itemize}[leftmargin=1.6em,itemsep=3pt,topsep=3pt]
\item \emph{Every interval is non-empty.} For $j>k$ this is automatic. For $j\le k$,
$(\theta-d_j)^+\le\beta_j\iff\theta\le\beta_j+d_j=\alpha_j$. So this condition is $\theta\le U_2$.
\item \emph{The upper ends are large enough:} $\theta\le\sum_j\overline\theta_j=\sum_{j\le k}\beta_j+\sum_{j>k}c_j=U_1$.
\item \emph{The lower ends are small enough:} $\sum_j\underline\theta_j=\sum_{j\le k}(\theta-d_j)^+\le\theta$. Since
$\sum_{j\le k}(x_j)^+=\max_{I\subseteq[k]}\sum_{j\in I}x_j$, this says
$|I|\,\theta-\sum_{j\in I}d_j\le\theta$, i.e.\ $(|I|-1)\,\theta\le\sum_{j\in I}d_j$, for every $I\subseteq[k]$.
This is automatic for $|I|\le1$, and for $|I|\ge2$ it is $\theta\le U_3$.
\end{itemize}
The three conditions together are $\theta\le\min\{U_1,U_2,U_3\}=U$.
\end{proof}


\begin{lemma}[Attainability]\label{lem:pns-attain}
There exists an admissible system with $\theta=U$.
\end{lemma}

\begin{proof}
Set $\theta:=U$ (a target the construction must hit, not an assumption).

\emph{Step 1 (how much of $S$ lies in each column).} By Lemma~\ref{lem:pns-alloc},
choose $\theta_1,\dots,\theta_n$ with
$\theta_\ell\in[\underline\theta_\ell,\overline\theta_\ell]$ and $\sum_\ell\theta_\ell=U$.

\emph{Step 2 (how much of each $A_j$ lies in each column).} For $j\le k$ let
$m_{jj}:=\beta_j$ and choose $m_{j\ell}\in[\theta_\ell,c_\ell]$ for $\ell\ne j$, summing
to $d_j$. This is possible because each interval is non-empty
($\theta_\ell\le\overline\theta_\ell\le c_\ell$), the lower ends are small enough
($\sum_{\ell\ne j}\theta_\ell=U-\theta_j\le d_j$, since
$\theta_j\ge\underline\theta_j=(U-d_j)^+$), and the upper ends are large enough
($\sum_{\ell\ne j}c_\ell=1-c_j\ge d_j$ by \eqref{eq:star}). Moreover
$m_{jj}=\beta_j=\overline\theta_j\ge\theta_j$. Hence $\theta_\ell\le m_{j\ell}\le c_\ell$
for all $j,\ell$, and $\sum_\ell m_{j\ell}=\beta_j+d_j=\alpha_j$.

\emph{Step 3 (arrangement inside each column).} In column $\ell$ place a block $K_\ell$
of mass $\theta_\ell$. For $\ell\le k$ put $K_\ell$ inside $Q_{\ell\ell}$ (possible since
$\theta_\ell\le\beta_\ell$) and set $A_\ell\cap x_\ell:=Q_{\ell\ell}$. For every
$j\ne\ell$, let $A_j$ contain $K_\ell$, and on $x_\ell\setminus K_\ell$ let $A_j$ occur
independently with probability $(m_{j\ell}-\theta_\ell)/(c_\ell-\theta_\ell)\in[0,1]$.
Then $\Prob(A_j\cap x_\ell)=m_{j\ell}$ for all $j,\ell$, so $\Prob(A_j)=\alpha_j$ and
$A_j\cap x_j=Q_{jj}$: the system is admissible.

Every $A_j$ contains every block, so $\theta\ge\sum_\ell\theta_\ell=U$. By
Corollary~\ref{cor:pns-sound}, $\theta\le U$. Hence $\theta=U$.
\end{proof}

By Lemma~\ref{lem:principle}, $\max\Theta=U$, so the upper bound of $\PNS(k)$ is sharp.


\subsection{The lower bound proof}\label{sec:lower}

For the upper bound we could build a maximiser directly, by putting as much mass as
possible into the intersection. The lower bound is harder. If we push $A_1,\dots,A_k$
apart in one column, their mass has to go to other columns, where it may overlap with
another event, so we cannot treat the columns one at a time. We therefore take a dual
approach. In Section~\ref{sec:frechet}, Fr\'echet bounds in each column turn the problem
into a flow problem. Section~\ref{sec:flow} then uses max-flow/min-cut to write the
minimum as a maximum over cuts, and Section~\ref{sec:lb-proof} shows that only $k+2$ of
these cuts can be positive.

For $k=1$, $\theta=\alpha_1$ is identified and $L=U=\alpha_1$. So we will only consider cases when $k\ge2$ below.

\subsubsection{Column-wise Fr\'echet}\label{sec:frechet}

Recall that the column $x_\ell$ has mass $c_\ell=\Prob(x_\ell)=\sum_ib_{i\ell}$, the
total mass of its cells $Q_{i\ell}=(x_\ell,y_i)$, $i\in[n]$.

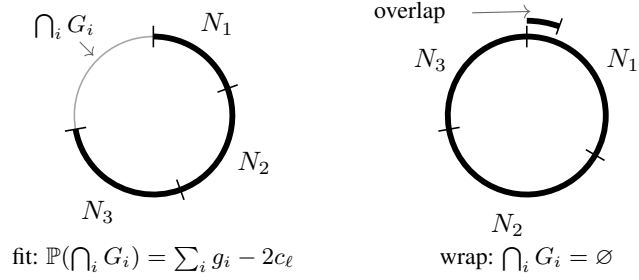
\begin{figure}[!t]
\centering
\begin{tikzpicture}[scale=0.95,
  ring/.style={black!35, line width=0.6pt},
  arcN/.style={black, line width=2.2pt},
  tick/.style={black, line width=0.6pt}]
  \begin{scope}
    \draw[ring] (0,0) circle (1.1);
    \draw[arcN] (90:1.1) arc (90:20:1.1);
    \draw[arcN] (20:1.1) arc (20:-70:1.1);
    \draw[arcN] (-70:1.1) arc (-70:-170:1.1);
    \foreach \a in {90,20,-70,-170} \draw[tick] (\a:0.95) -- (\a:1.25);
    \node at (55:1.55) {$N_1$};
    \node at (-25:1.55) {$N_2$};
    \node at (-120:1.55) {$N_3$};
    \node[font=\small] at (135:1.75) {$\bigcap_iG_i$};
    \draw[->,black!60] (135:1.45) -- (135:1.18);
    \node[font=\small] at (0,-2.0) {fit: $\Prob(\bigcap_iG_i)=\sum_ig_i-2c_\ell$};
  \end{scope}
  \begin{scope}[xshift=5.2cm]
    \draw[ring] (0,0) circle (1.1);
    \draw[arcN] (90:1.1) arc (90:-30:1.1);
    \draw[arcN] (-30:1.1) arc (-30:-170:1.1);
    \draw[arcN] (-170:1.1) arc (-170:-270:1.1);
    \draw[arcN] (90:1.32) arc (90:70:1.32);
    \foreach \a in {90,-30,-170} \draw[tick] (\a:0.95) -- (\a:1.25);
    \draw[tick] (70:1.2) -- (70:1.44);
    \node[font=\small,anchor=east] at (125:1.75) {overlap};
    \draw[->,black!60] (118:1.62) -- (88:1.45);
    \node at (30:1.55) {$N_1$};
    \node at (-100:1.55) {$N_2$};
    \node at (150:1.55) {$N_3$};
    \node[font=\small] at (0,-2.0) {wrap: $\bigcap_iG_i=\varnothing$};
  \end{scope}
\end{tikzpicture}
\caption{Lemma~\ref{lem:frechet} in one column of mass $c_\ell$, with $t=3$. We place the
complements $N_i$ one after another on a circle of circumference $c_\ell$. If they fit,
they do not overlap and the union bound holds with equality. If they do not fit,
they wrap around and cover the whole column, so the intersection is empty.}
\label{fig:frechet}
\end{figure}

\begin{lemma}[Column-wise Fr\'echet \citep{frechet1951tableaux}]\label{lem:frechet}
Fix a column $x_\ell$ and let $G_1,\dots,G_t\subseteq x_\ell$ be events. Then
$\Prob(\bigcap_iG_i)\ge\sum_i\Prob(G_i)-(r-1)\,c_\ell$. Conversely, for any masses
$g_i\in[0,c_\ell]$ there is an arrangement inside $x_\ell$ with $\Prob(G_i)=g_i$ and
$\Prob(\bigcap_iG_i)=(\sum_ig_i-(r-1)\,c_\ell)^+$. If some $G_i$ must be a union of
cells, this is still possible, and arrangements in different columns can be chosen
independently.
\end{lemma}
\begin{proof}
Let $N_i:=x_\ell\setminus G_i$ be the complement of $G_i$ inside the column. By De Morgan
and the union bound,
$\Prob(\bigcap_iG_i)=c_\ell-\Prob(\bigcup_iN_i)\ge c_\ell-\sum_i(c_\ell-g_i)$. For the
converse, view the column as a circle of circumference $c_\ell$ and lay $N_1,\dots,N_r$
end to end as arcs of lengths $c_\ell-g_i$ (Figure~\ref{fig:frechet}). If they fit, they
are disjoint and the bound is an equality. Otherwise they wrap around and together cover
the whole circle, so every point lies in some $N_i$ and $\bigcap_iG_i=\varnothing$. If
some $G_i$ must equal a union of cells (for instance $A_j\cap x_j=(x_j,y_j)$ when
$j\le k$), its complement $N_i$ is the union of the remaining cells of the column; since
only the masses of the cells are prescribed, we place these cells side by side on the
arc assigned to $N_i$, whose length equals their total mass. Finally, distinct columns
are disjoint, so the arrangements in different columns do not interact and can be chosen
simultaneously.
\end{proof}

By pigeonhole the complements fit disjointly iff the bracket is non-negative, so
$(\cdot)^+$ is not an artificial truncation. As in Section~\ref{sec:upper}, the
inequality holds for every arrangement and gives the bound, and a suitable arrangement
attains it.

\subsubsection{Flow and cuts}\label{sec:flow}

For the lower bounds with $k\ge2$ we use the transportation set
$\cT:=\{m\in\mathbb R^{k\times n}:m_{jj}=\beta_j,\ \sum_{\ell\ne j}m_{j\ell}=d_j,\ 0\le m_{j\ell}\le c_\ell\}$.
Recall that $d_j=\alpha_j-\beta_j=\Prob(A_j\setminus x_j)=\Prob(A_j)-\Prob(A_j, x_j)$
is the mass of $A_j$ outside its own column $x_j$. In the table $m$, row $j$ describes
how $A_j$ is spread over the columns: its diagonal entry is $m_{jj}=\beta_j$, and its
off-diagonal entries $m_{j\ell}$, $\ell\ne j$, must add up to exactly $d_j$. We also use
the column inflow $e_\ell:=\sum_{j\le k,\,j\ne\ell}m_{j\ell}$, the mass that the other
rows place in column $\ell$; summing over all off-diagonal cells by rows or by columns
gives $D:=\sum_\ell e_\ell=\sum_jd_j$.

In the proof of the lower bound, Lemma~\ref{lem:frechet} will give, for each column
$\ell$, a lower bound of the form $(e_\ell-f_\ell)^+$, where $f_\ell\ge0$ depends only on
the data; for $\PNS(k)$ it is $f_\ell=(k-1)c_\ell-\beta_\ell\one\{\ell\le k\}$. We call
$f_\ell$ the \emph{free capacity} of column $\ell$: inflow up to $f_\ell$ costs nothing,
and each unit beyond it adds one to $\theta$. The lower bound is therefore a
minimisation of $\sum_\ell(e_\ell-f_\ell)^+$ over the transportation set. The next lemma,
stated for arbitrary $f_1,\dots,f_n\ge0$, rewrites this minimisation as a maximisation
over the $2^k$ subsets $J\subseteq[k]$.

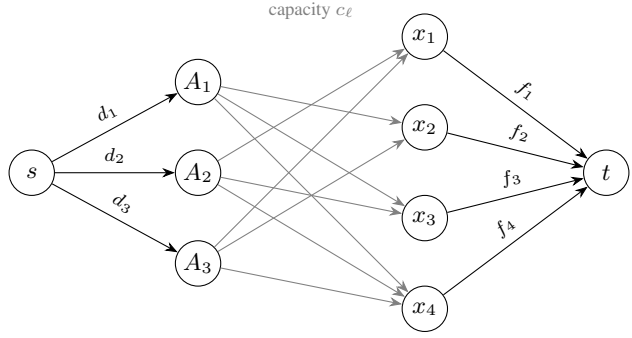
\begin{figure}[!t]
\centering
\begin{tikzpicture}[>=Stealth, node distance=1cm,
  nd/.style={circle,draw,minimum size=6mm,inner sep=0pt,font=\small}]
  \node[nd] (s) at (0,0) {$s$};
  \foreach \j/\y in {1/1.2,2/0,3/-1.2} \node[nd] (r\j) at (2.2,\y) {$A_\j$};
  \foreach \l/\y in {1/1.8,2/0.6,3/-0.6,4/-1.8} \node[nd] (c\l) at (5.2,\y) {$x_\l$};
  \node[nd] (t) at (7.6,0) {$t$};
  \foreach \j in {1,2,3} \draw[->] (s) -- node[above,sloped,font=\scriptsize] {$d_\j$} (r\j);
  \foreach \j in {1,2,3}{\foreach \l in {1,2,3,4}{
      \ifnum\j=\l\else\draw[->,gray] (r\j) -- (c\l);\fi}}
  \foreach \l in {1,2,3,4} \draw[->] (c\l) -- node[above,sloped,font=\scriptsize] {$f_\l$} (t);
  \node[font=\scriptsize,gray] at (3.7,2.15) {capacity $c_\ell$};
\end{tikzpicture}
\caption{The network of Lemma~\ref{lem:flow} for $k=3$, $n=4$. Each $A_j$ must place
mass $d_j=\Prob(A_j\setminus x_j)$ in the columns $x_\ell$, $\ell\ne j$, at most $c_\ell$ in
each. There is no edge from $A_j$ to $x_j$, because inside its own column $A_j$ is fixed
to the cell of mass $\beta_j$. Column $x_\ell$ can absorb up to $f_\ell$ for free, and any inflow beyond that adds to $\theta$.}
\label{fig:network}
\end{figure}

\begin{lemma}[Max-flow/min-cut]\label{lem:flow}
Let $f_1,\dots,f_n\ge0$. For a set of rows $J\subseteq[k]$, put
\[
\Delta(J):=\sum_{j\in J}d_j-\sum_{\ell=1}^n\min\big(|J\setminus\{\ell\}|\times c_\ell,\ f_\ell\big).
\]
Then
\[
\min_{m\in\cT}\ \sum_{\ell=1}^n\big(e_\ell-f_\ell\big)^+=\max_{J\subseteq[k]}\Delta(J).
\]
\end{lemma}

Here $J$ is any set of rows, i.e.\ of targets $A_j$. The rows in $J$ must send out
$\sum_{j\in J}d_j$ in total, while column $\ell$ can absorb free of cost at most
$\min(|J\setminus\{\ell\}|\times c_\ell,f_\ell)$ of it: at most $c_\ell$ from each row
$j\in J$ with $j\ne\ell$, and at most $f_\ell$ altogether. So $\Delta(J)$ is the part of
the rows in $J$ that cannot be absorbed for free, and each $\Delta(J)$ is a lower bound
on the cost. In the proof, $J$ appears as the set of rows on the source side of a cut.

\begin{proof}
Consider the network of Figure~\ref{fig:network}, in which the source feeds row $j$ with
capacity $d_j$, row $j$ feeds column $\ell\ne j$ with capacity $c_\ell=\Prob(x_\ell)$
(since $m_{j\ell}=\Prob(A_j\cap x_\ell)\le c_\ell$), and column $\ell$ feeds the sink with
capacity $f_\ell$. Let $\varphi$ be the value of a maximum flow.
Since
$(e-f)^+=e-\min(e,f)$ and $\sum_\ell e_\ell=D$, every $m\in\cT$ satisfies
$\sum_{\ell}(e_\ell-f_\ell)^+ = D-\sum_\ell\min(e_\ell,f_\ell)$.

\emph{$\min_{m\in\cT}\sum_\ell(e_\ell-f_\ell)^+\ge D-\varphi$: 
every $m\in\cT$ has cost
at least $D-\varphi$.} Fix $m\in\cT$. We turn it into a flow by letting row $j$ send to
column $\ell\ne j$ the amount
\[
\eta_{j\ell}:=m_{j\ell}\,\lambda_\ell,\qquad
\lambda_\ell:=\begin{cases}\min(e_\ell,f_\ell)/e_\ell,& e_\ell>0,\\ 0,& e_\ell=0,\end{cases}
\]
that is, the inflows of column $\ell$ are scaled down by the common factor
$\lambda_\ell\in[0,1]$. This flow respects every capacity: row $j$ ships
$\sum_{\ell\ne j}\eta_{j\ell}\le\sum_{\ell\ne j}m_{j\ell}=d_j$; on each arc
$\eta_{j\ell}\le m_{j\ell}\le c_\ell$; and column $\ell$ receives
$\sum_{j\ne\ell}\eta_{j\ell}=\lambda_\ell e_\ell=\min(e_\ell,f_\ell)\le f_\ell$. Its value
is $\sum_\ell\min(e_\ell,f_\ell)$, and $\varphi$ is the largest value of any flow, so
\[
\sum_\ell(e_\ell-f_\ell)^+=D-\sum_\ell\min(e_\ell,f_\ell)\ge D-\varphi.
\]

\emph{$\min_{m\in\cT}\sum_\ell(e_\ell-f_\ell)^+\le D-\varphi$: 
some $m\in\cT$ has cost
at most $D-\varphi$.} Take a maximum flow, and let $\eta_{j\ell}\ge0$ be the amount it
sends from row $j$ to column $\ell\ne j$. Write
\[
\eta_j:=\sum_{\ell\ne j}\eta_{j\ell},\qquad\qquad
\varphi_\ell:=\sum_{j\ne\ell}\eta_{j\ell}
\]
for the total amount row $j$ ships and the total amount column $\ell$ receives, respectively. So that $\sum_\ell\varphi_\ell=\varphi$, and the
capacities give $\eta_j\le d_j$, $\eta_{j\ell}\le c_\ell$ and $\varphi_\ell\le f_\ell$.

\emph{Completing the flow to a table $m\in\cT$.} Row $j$ must place $d_j$ in the columns
$\ell\ne j$, and the flow already places $\eta_j$ there. The room left in its cells is
\[
\sum_{\ell\ne j}(c_\ell-\eta_{j\ell})=(1-c_j)-\eta_j\ge d_j-\eta_j
\]
by \eqref{eq:star}.
So we can choose $m_{j\ell}\in[\eta_{j\ell},c_\ell]$ for $\ell\ne j$ with
$\sum_{\ell\ne j}m_{j\ell}=d_j$; together with $m_{jj}:=\beta_j$ this gives $m\in\cT$.

\emph{Bounding its cost.} Since $m_{j\ell}\ge\eta_{j\ell}$, the inflow of column $\ell$
satisfies $e_\ell\ge\varphi_\ell$, and $\varphi_\ell\le f_\ell$. Hence
$(e_\ell-f_\ell)^+\le e_\ell-\varphi_\ell$: if $e_\ell\ge f_\ell$, then
$e_\ell-f_\ell\le e_\ell-\varphi_\ell$; if $e_\ell<f_\ell$, then $0\le e_\ell-\varphi_\ell$.
Summing over $\ell$, 
\[
\sum_\ell(e_\ell-f_\ell)^+\ \le\ \sum_\ell e_\ell-\sum_\ell\varphi_\ell\ =\ D-\varphi .
\]

\emph{Computing $\varphi$ by a minimum cut.} A cut splits the nodes into a source side
and a sink side, and its capacity is the total capacity of the arcs that go from the
source side to the sink side. Let $J\subseteq[k]$ be the set of rows on the source side.
\begin{itemize}
\item Each row $j\notin J$ is on the sink side, so the arc from the source to row $j$ is cut,
at cost $d_j$. 
\item Each column $\ell$ is placed on the cheaper side. On the source side, its
arc to the sink is cut, at cost $f_\ell$. On the sink side, the arcs into it from the
rows of $J$ are cut; there are $|J\setminus\{\ell\}|$ of them (the number of elements of
$J\setminus\{\ell\}$), each of capacity $c_\ell$, at total cost
$|J\setminus\{\ell\}|\times c_\ell$. 
\end{itemize}

The cheapest cut with row set $J$ therefore has
capacity $\sum_{j\notin J}d_j+\sum_\ell\min(|J\setminus\{\ell\}|\times c_\ell,f_\ell)$. By
the max-flow/min-cut theorem \citep{ford1956maximal}, $\varphi$ equals the minimum of
this over $J$. Since $D-\sum_{j\notin J}d_j=\sum_{j\in J}d_j$,
\[
D-\varphi=\max_{J\subseteq[k]}\Big[\sum_{j\in J}d_j-\sum_\ell\min\big(|J\setminus\{\ell\}|\times c_\ell,\ f_\ell\big)\Big]
=\max_{J\subseteq[k]}\Delta(J).
\]
\end{proof}

Lemma~\ref{lem:flow} expresses the lower bound as a maximum over all $2^k$ sets
$J\subseteq[k]$, whereas the published bounds list three expressions; the gap is closed
by \eqref{eq:star} alone. We now show that, when $f_\ell\ge(k-2)\,c_\ell$ for every
$\ell$, every set $J$ with $|J|\le k-2$ gives $\Delta(J)\le0$. Since
$\Delta(\varnothing)=0$, such sets never attain the maximum, and it suffices to consider
the $k+2$ sets $J=\varnothing$, $J=[k]\setminus\{i\}$ ($i\le k$) and $J=[k]$.

\begin{lemma}[Small cuts]\label{lem:small}
If $f_\ell\ge(k-2)\,c_\ell$ for every $\ell$ and $|J|\le k-2$, then
$\Delta(J)=-\sum_{j\in J}s_j\le0$.
\end{lemma}

\begin{proof}
Let $|J|\le k-2$. For every column $\ell$ we have $|J\setminus\{\ell\}|\le|J|\le k-2$,
so $|J\setminus\{\ell\}|\times c_\ell\le(k-2)\,c_\ell\le f_\ell$, and each minimum in
$\Delta(J)$ equals $|J\setminus\{\ell\}|\times c_\ell$. Since $|J\setminus\{\ell\}|=|J|-1$
for $\ell\in J$ and $|J\setminus\{\ell\}|=|J|$ for $\ell\notin J$, and $\sum_\ell c_\ell=1$,
\[
\sum_{\ell=1}^n|J\setminus\{\ell\}|\times c_\ell=|J|-\sum_{j\in J}c_j.
\]
Hence
\[
\Delta(J)=\sum_{j\in J}d_j-|J|+\sum_{j\in J}c_j=\sum_{j\in J}(d_j+c_j-1)=-\sum_{j\in J}s_j\le0,
\]
where the last inequality is by applying \eqref{eq:star}.
\end{proof}

\subsubsection{Proof of the lower bound}\label{sec:lb-proof}
Now, let's finish the proof of the sharpness of lower bound.

\emph{Reduction to a flow.} By Lemma~\ref{lem:frechet} in column $\ell$, applied to the
events $A_j\cap x_\ell$ ($j\le k$),
\[
\theta_\ell\ \ge\ \big(\textstyle\sum_{j}m_{j\ell}-(k-1)c_\ell\big)^+=(e_\ell-f_\ell)^+,\qquad
f_\ell:=(k-1)c_\ell-\beta_\ell\one\{\ell\le k\}.
\]
Equality holds for a suitable arrangement, simultaneously in all columns. Hence
$\min\Theta=\min_{m\in\cT}\sum_\ell(e_\ell-f_\ell)^+$, where every $m\in\cT$ is realised
by some admissible system. We have $f_\ell\ge(k-2)c_\ell\ge0$ because
$\beta_\ell\le c_\ell$. By Lemmas~\ref{lem:flow} and~\ref{lem:small},
$\min\Theta=\max\{0,\Delta([k]\setminus\{i\})_{i\le k},\Delta([k])\}$.

\emph{Evaluating the two cut families.} By Lemma~\ref{lem:small}, every $J$ with $|J|\le k-2$ has $\Delta(J)\le0$, so apart from
$\Delta(\varnothing)=0$, only sets with at least $k-1$ elements can give a positive value. These are $J=[k]\setminus\{i\}$ for $i\le k$, and $J=[k]$. Recall that
$\Delta(J)=\sum_{j\in J}d_j-\sum_{\ell}\min\big(|J\setminus\{\ell\}|\times c_\ell,\ f_\ell\big)$
with $f_\ell=(k-1)c_\ell-\beta_\ell\one\{\ell\le k\}$, and that $\sum_\ell c_\ell=1$.
\begin{itemize}[leftmargin=*,itemsep=1pt,topsep=2pt]
\item $J=[k]\setminus\{i\}$ for some $i\le k$. We compute the minimum column by column: 
\begin{itemize}
\item for $\ell=i$,
$|J\setminus\{\ell\}|=k-1$ and the minimum is $(k-1)c_i-\beta_i=f_i$; \item for $\ell\le k$,
$\ell\ne i$, $|J\setminus\{\ell\}|=k-2$ and $f_\ell\ge(k-2)c_\ell$, so the minimum is
$(k-2)c_\ell$; 
\item for $\ell>k$, $|J\setminus\{\ell\}|=k-1$ and $f_\ell=(k-1)c_\ell$, so the
minimum is $(k-1)c_\ell$. 
\end{itemize}
Adding these up gives
$(k-1)\sum_\ell c_\ell-\sum_{j\le k,\,j\ne i}c_j-\beta_i=(k-1)-\sum_{j\le k,\,j\ne i}c_j-\beta_i$.

Hence, using $d_j=\alpha_j-\beta_j$,
\[
\Delta([k]\setminus\{i\})=\beta_i+\sum_{j\le k,\,j\ne i}(d_j+c_j)-(k-1)
=\beta_i+\sum_{j\le k,\,j\ne i}(\alpha_j+c_j-\beta_j)-(k-1).
\]
\item $J=[k]$. We also compute column by column:
\begin{itemize}
\item for $\ell\le k$, $|J\setminus\{\ell\}|=k-1$ and the minimum is
$(k-1)c_\ell-\beta_\ell=f_\ell$; 
\item for $\ell>k$, $|J\setminus\{\ell\}|=k$ and the minimum is
$\min(kc_\ell,(k-1)c_\ell)=(k-1)c_\ell$. 
\end{itemize}
Adding these up gives
$(k-1)-\sum_{j\le k}\beta_j$, hence
\[
\Delta([k])=\sum_{j\le k}d_j-(k-1)+\sum_{j\le k}\beta_j=\sum_{j\le k}\alpha_j-(k-1).
\]
\end{itemize}
Together with $\Delta(\varnothing)=0$, these are exactly the three terms of $L$:
$\Delta(\varnothing)=0=L_1$, $\Delta([k])=L_2$, and
$\max_{i\le k}\Delta([k]\setminus\{i\})=L_3$. Thus $\min\Theta=L$, and with
Section~\ref{sec:upper}, Theorem~\ref{thm:pns} is proved. 

\subsection{The other three families}\label{sec:others}

The same argument works for targets of the form $\Prob(S,G)$ when $G\ne\Omega$. By Lemma~\ref{lem:reduction}, we apply admissible systems again. We restate the bounds of \citet{Shu_2026} in our notation, 
and give the full proofs in Appendix. Throughout, we use the canonical form, in which $p>k$.

\paragraph{$\PSub(k,p)$ and $\PN(k,p,q)$.} 
$G=x_p$ for $\PSub$, and $G=(x_p,y_q)=Q_{qp}\subseteq x_p$ for $\PN$. We write
$\gamma:=\Prob(G)$, so $\gamma=c_p$ for $\PSub$ and $\gamma=b_{qp}$ for
$\PN$ by consistency. The target is $\theta=\Prob(S,G)$, and the bounds are
\begin{equation}\label{eq:psub}
\max\Big\{0,\ \sum_{j\le k}(\alpha_j+c_j-\beta_j)+\gamma-k\Big\}\ \le\ \theta\ \le\
\min\Big\{\gamma,\ \min_{j\le k}d_j\Big\}.
\end{equation}

\paragraph{$\PRep(k,q)$.} The target is $\theta=\Prob(S,y_q)$. By consistency, only the
cells of row $q$ in column $q$ (if $q\le k$) and in the columns $\ell>k$ can meet $S$; their
total mass is
\[
T:=\beta_q\one\{q\le k\}+\sum_{\ell>k}b_{q\ell},
\]
and the bounds are
\begin{align}
\theta&\ \ge\ \max\Big\{0,\ \sum_{j\le k}(\alpha_j+c_j-\beta_j)+T-k,\
\one\{q\le k\}\Big[\beta_q+\!\!\sum_{j\le k,\,j\ne q}\!\!(\alpha_j+c_j-\beta_j)-k+1\Big]\Big\},\label{eq:prep-l}\\
\theta&\ \le\ \min\Big\{T,\ \ \alpha_q\ \text{(if $q\le k$)},\ \ \min_{j\le k,\,j\ne q}d_j\Big\}.\label{eq:prep-u}
\end{align}


\begin{theorem}\label{thm:others}
For any compatible data and any admissible $p$ and $q$, the bounds \eqref{eq:psub} on
$\PSub(k,p)$ and $\PN(k,p,q)$ and the bounds \eqref{eq:prep-l}--\eqref{eq:prep-u} on
$\PRep(k,q)$ are sharp for every $k\ge1$.
\end{theorem}

For $\PRep(k,q)$ with $k=1$, the max-flow argument of Section~\ref{sec:lower} does not apply, since the free capacity $f_q=-\beta_q$ is negative, but sharpness still holds and follows from Theorem~\ref{thm:unified}. The full proofs of sharpness for these three families are given in the appendix.


So far each target form pairs $x_j$ with $y_j$. A general query may pair $x_j$ with any outcome,
for example $\Prob(\yx{3}{1},\yx{1}{2})$. \citet{Shu_2026} obtain its bounds from the
canonical ones by substitution: for a target $\yx{i}{j}$, they use $\Prob(\yx{i}{j})$ in
place of $\Prob(\yx{j}{j})$ and $\Prob(x_j,y_i)$ in place of $\Prob(x_j,y_j)$.

\begin{proposition}\label{prop:relabel}
Substitution gives the sharp bounds for $\PNS(k)$, $\PSub(k,p)$ and $\PN(k,p,q)$. For
$\PRep(k,q)$, it gives the sharp bounds when at most one target has the form $\yx{q}{j}$. Theorem~\ref{thm:unified} covers the remaining case, in which more than one target has
the form $\yx{q}{j}$, for example when both $\yx{q}{1}$ and $\yx{q}{2}$ appear.
\end{proposition}

A complete proof follows from Theorem~\ref{thm:unified}; see Appendix.


\section{Unified results}\label{sec:unified}
In this section, we extend our main result and summarize the four families into one theorem.

The sharp bounds of the four families (Theorem~\ref{thm:pns} for $\PNS(k)$ and
\eqref{eq:psub}--\eqref{eq:prep-u} for the other three), are four instances of a single
statement. The four families differ only in the factual event $G$, and the proofs above
use $G$ in one place: through the part of $G$ in each column that can still meet $S$ once
consistency is taken into account. The statement below also allows arbitrary target
outcomes, so it covers relabelled queries as well.

\paragraph{The general notation.} Here, we keep $G$ to be a factual event, and change the target outcome to $i_1,\dots,i_k\in[n]$, we allow $i_u = i_v$ for $u,v \in [k]$ and $u\ne v$. We also change the notation of $A_j$ we defined in Section~\ref{sec:setup} as
\[
A_j:=\yx{i_j}{j}.
\]
According to this new $A_j$, we have 
\[
\qquad S:=\textstyle\bigcap_{j\le k}A_j,\qquad \theta:=\Prob(S, G),\qquad
\alpha_j:=a_{i_jj},\qquad \beta_j:=b_{i_jj},
\]
and keep $d_j=\alpha_j-\beta_j$, $c_j=\Prob(x_j)$ and $s_j=1-c_j-d_j$. 

With the new $\alpha_j$ and $\beta_j$, we have $s_j=(1-a_{i_jj})-(c_j-b_{i_jj})$, which involves the other rows $i\ne i_j$ of column $j$. By consistency, $b_{ij}\le a_{ij}$ for every $i$ and $j$, summing over these rows gives $s_j=\sum_{i\ne i_j}(a_{ij}-b_{ij})\ge0$, so \eqref{eq:star} still holds. The proof of Lemma~\ref{lem:reduction} also remains valid with the target outcome $j$ replaced by $i_j$ throughout, so $\Theta$ is still the set of values of admissible systems with the change of notation $A_j\cap x_j=Q_{i_j\,j}$.

For a target column $x_j$ with $j\le k$. Every
unit in this column received $x_j$, so by consistency its outcome under $x_j$ is simply its
observed outcome. We let $E_j$ denote when $A_j$ holds, and let $F_j$ denote when $A_j$ fails on the rest of the column,
\[
E_j:=Q_{i_j\,j}=(x_j,y_{i_j}),\qquad F_j:=x_j\setminus E_j.
\]
Let
\[
C:=G\setminus\textstyle\bigcup_{j\le k}F_j .
\]
Then $\theta$ can be rewritten as $\theta=\Prob(S,C)$, since no unit in $F_j$ can be in $S$, so we can drop $F_j$ from $G$ and the value of $\theta$ stays the same. Write 
\[
\gamma_\ell:=\Prob(C, x_\ell)
\]
for the mass of $C$ in column $\ell$, and $\gamma:=\Prob(C)=\sum_\ell\gamma_\ell$. In a target column only $E_j$
can remain, so $\gamma_j=\beta_j$ if $E_j\subseteq G$ and $\gamma_j=0$ otherwise. In a column $\ell>k$ nothing is dropped, $C = G$, so $\gamma_\ell$ is just the mass of $G$ in that column.

\begin{theorem}[Unified form]\label{thm:unified}
In every compatible pair of tables, for every $k\ge1$, all target outcomes $i_1,\dots,i_k$ and every factual event $G$, $\Theta=[L_G,U_G]$, where
\begin{align*}
L_G&=\max\Big\{\ \underbrace{0}_{L_1},\ \
\underbrace{\gamma-\textstyle\sum_{j\le k}s_j}_{L_2},\ \
\underbrace{\max_{i\le k}\big(\gamma_i-\textstyle\sum_{j\le k,\,j\ne i}s_j\big)}_{L_3}\ \Big\},\\
U_G&=\min\Big\{\ \underbrace{\gamma}_{U_1},\ \
\underbrace{\min_{j\le k}\,(d_j+\gamma_j)}_{U_2},\ \
\underbrace{\min_{I\subseteq[k],\,|I|\ge2}\tfrac{1}{|I|-1}\textstyle\sum_{j\in I}d_j}_{U_3}\ \Big\}.
\end{align*}
When $k=1$, the term $U_3$ is absent and $L_3=\gamma_1$.
\end{theorem}

The theorem depends on $G$ only through the vector $(\gamma_1,\dots,\gamma_n)$, which
records where $G$ survives consistency. For the canonical families ($i_j=j$, $p>k$):
\begin{center}\small
\begin{tabular}{@{}lllll@{}}
\toprule
family & $G$ & $\gamma_j$ ($j\le k$) & $\gamma_\ell$ ($\ell>k$) & $\gamma=\Prob(C)$\\
\midrule
$\PNS(k)$ & $\Omega$ & $\beta_j$ & $c_\ell$ & $\sum_{j\le k}\beta_j+\sum_{\ell>k}c_\ell$\\
$\PSub(k,p)$ & $x_p$ & $0$ & $c_p\one\{\ell=p\}$ & $c_p$\\
$\PN(k,p,q)$ & $(x_p,y_q)$ & $0$ & $b_{qp}\one\{\ell=p\}$ & $b_{qp}$\\
$\PRep(k,q)$ & $y_q$ & $\beta_q\one\{j=q\}$ & $b_{q\ell}$ & $T=\beta_q\one\{q\le k\}+\sum_{\ell>k}b_{q\ell}$\\
\bottomrule
\end{tabular}
\end{center}
For $\PSub$ and $\PN$, $\gamma$ is the $\gamma=\Prob(G)$ of \eqref{eq:psub}, and for
$\PRep$ it is $T$.

Substituting these values into Theorem~\ref{thm:unified} and using $\sum_\ell c_\ell=1$
recovers four families.
For $\PSub$, $\PN$ and $\PRep$ the term $U_3$ is redundant, since an average of the $d_j$
over $I$ is never smaller than $\min_jd_j$, and every $L_3$ term with $\gamma_i=0$ is
non-positive.


The proof of Theorem~\ref{thm:unified} is given in Appendix.

\section{Conclusion}\label{sec:conclusion}
We have proved the conjecture of \citet{Shu_2026}: their closed-form bounds are sharp in all dimensions. The proof also explains why the bounds take the form they do. The upper bounds are exactly the conditions under which the mass of $S$ can be spread over the
columns, and the lower bounds come from the only cuts that can bind. We further derived single formula that contains all four families as special cases. Together with the equivalence classes of \citet{Shu_2026}, it gives sharp bounds for any discrete PoC in
every dimension. Our results use only experimental and observational data, with no further assumptions. Additional assumptions such as monotonicity, or extra information such as covariates could tighten the bounds and can be a future research direction.

\bibliographystyle{plainnat}
\bibliography{refs}

\end{document}